\documentclass[11pt,a4paper]{article}
\usepackage{times,latexsym}
\usepackage{url}
\usepackage[T1]{fontenc}

\usepackage[final]{emnlp2022}
\usepackage{xspace,mfirstuc,tabulary}

\usepackage{rycolab}
\usepackage{tetratagger}

\usepackage{pgfplots}
\usepackage{pgfplotstable}
\pgfplotsset{compat=1.7}

\usepackage{qtree}
\usepackage{tikz}
\usetikzlibrary{shapes,arrows,chains}
\usepackage{relsize}

\usepackage{circledsteps}
\pgfkeys{/csteps/inner color=gray}
\pgfkeys{/csteps/inner color=gray}

\usepackage{forest}
\usepackage{makecell}

\usepackage{microtype}

\usepackage{booktabs}
\usepackage{multirow}

\usepackage{inconsolata}
\usepackage{amssymb,graphicx}
\usepackage{tikz}
\usepackage{pgfplots}
\usepackage{color}
\usepackage{xcolor}

\title{On Parsing as Tagging}

\author{Afra Amini~\qquad~Ryan Cotterell \\
\fcolorbox{white}{white}{
 $\{$\texttt{\href{mailto:afra.amini.ethz.ch}{afra.amini}, }\texttt{\href{mailto:ryan.cotterell@inf.ethz.ch}{ ryan.cotterell}}$\}$\texttt{@inf.ethz.ch}
} \\ 
\setlength{\fboxsep}{2.5pt}
\setlength{\fboxrule}{2.5pt}
\fcolorbox{white}{white}{
    \includegraphics[width=.15\linewidth]{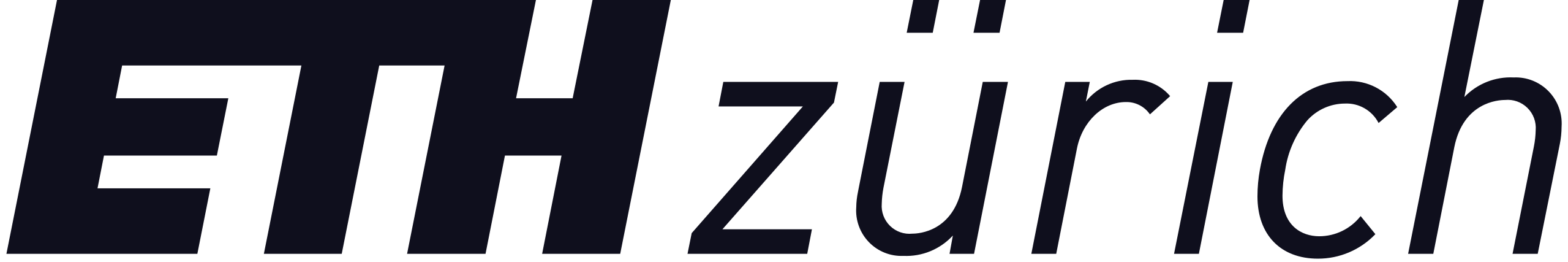}
}
}

\begin{document}
\maketitle
\begin{abstract}
There have been many proposals to reduce constituency parsing to tagging in the literature. 
To better understand what these approaches have in common, we cast several existing proposals into a unifying pipeline consisting of three steps: linearization, learning, and decoding. In particular, we show how to reduce tetratagging, a state-of-the-art constituency tagger, to shift--reduce parsing by performing a right-corner transformation on the grammar and making a specific independence assumption.
Furthermore, we empirically evaluate our taxonomy of tagging pipelines with different choices of linearizers, learners, and decoders. 
Based on the results in English and a set of 8 typologically diverse languages, we conclude that the linearization of the derivation tree and its alignment with the input sequence is the most critical factor in achieving accurate taggers.\looseness=-1

\vspace{1.5em}
\hspace{.5em}\includegraphics[width=1.25em,height=1.25em]{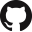}\hspace{.75em}\parbox{\dimexpr\linewidth-2\fboxsep-2\fboxrule}{\url{https://github.com/rycolab/parsing-as-tagging}}
\vspace{-1.0em}
\end{abstract}
\begin{figure}[t]
    \centering
     \setlength{\belowcaptionskip}{-10pt}
\includegraphics[width=\columnwidth]{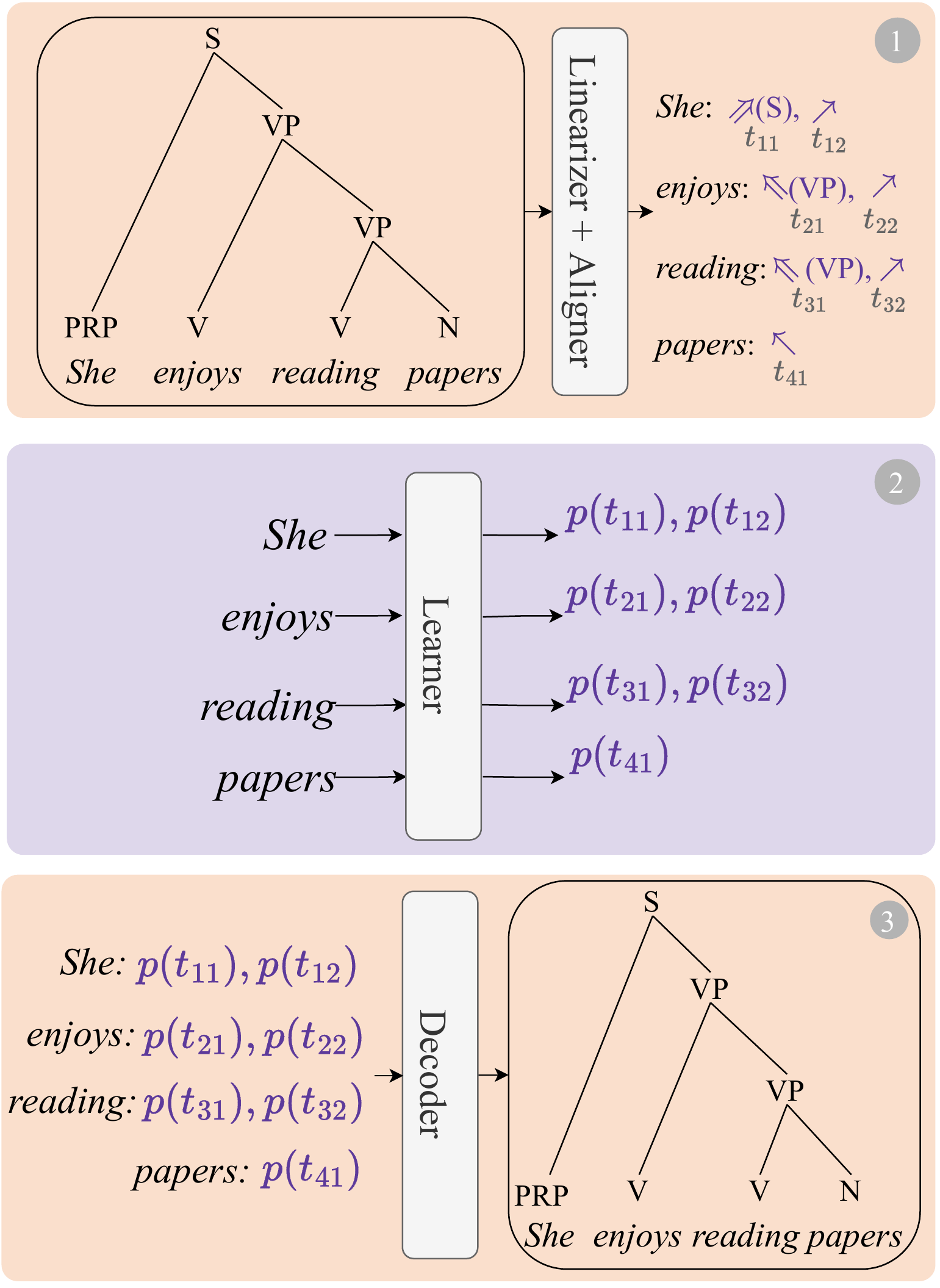}
    \caption{We subdivide parsing-as-tagging schemes into three steps: 1) linearization, 2) learning, and 3) decoding.
    The figure depicts how subdivision affects the parsing of the sentence \textit{She enjoys reading papers} with pre-order linearization.}
    \label{fig:pat-pipeline}
\end{figure}
\section{Introduction}
The automatic syntactic analysis of natural language text is an important problem in natural language processing (NLP).
Due to the ambiguity present in natural language,
many parsers for natural language are statistical in nature, i.e., they provide a probability distribution over syntactic analyses for a given input sequence.
A common design pattern for natural language parsers is to re-purpose tools from formal language theory that were often designed for compiler construction \citep{Aho:72}, to construct probability distributions over derivation trees (syntactic analyses).
Indeed, one of the most commonly deployed parsers in NLP is a statistical version of the classic shift--reduce parser, which has been widely applied in constituency parsing \citep{sagae-lavie-2005-classifier, zhang-clark-2009-transition, zhu-etal-2013-fast}, dependency parsing \citep{fraser1989parsing, yamada-matsumoto-2003-statistical, nivre-2004-incrementality}, and even for the parsing of more exotic formalisms, e.g., CCG parsing \citep{zhang-clark-2011-shift, xu-etal-2014-shift}.\looseness=-1

Another relatively recent trend in statistical parsing is the idea of reducing parsing to tagging \citep{gomez-rodriguez-vilares-2018-constituent, strzyz-etal-2019-viable, vilares-etal-2019-better, vacareanu-etal-2020-parsing, gomez-rodriguez-etal-2020-unifying, kitaev-klein-2020-tetra}. 
There are a number of motivations behind this reduction. 
For instance, \citet{kitaev-klein-2020-tetra} motivate tetratagging, their novel constituency parsing-as-tagging scheme, by its ability to parallelize the computation during training and to minimize task-specific modeling. 
The proposed taggers in the literature are argued to exhibit the right balance in the trade-off between accuracy and efficiency.
Nevertheless, there are several ways to cast parsing as a tagging problem, and the relationship of each of these ways to transition-based parsing remains underdeveloped.\looseness=-1

To better understand the relationship between parsing as tagging and transition-based parsing, we identify a common pipeline that many parsing-as-tagging approaches seem to follow. 
This pipeline, as shown in \cref{fig:pat-pipeline}, consists of three steps: linearization, learning, and decoding. 
Furthermore, we demonstrate that \citeposs{kitaev-klein-2020-tetra} tetratagging may be derived under two \emph{additional} assumptions on a classic shift--reduce parser: (i) a right-corner transform in the linearization step and (ii) factored scoring of the rules in the learning step.
We find that (i) leads to better alignment with the input sequence, while (ii) leads to parallel execution.\looseness=-1

In light of these findings, we propose a taxonomy of parsing-as-tagging schemes based on different choices of linearizations, learners, and decoders.
Furthermore, we empirically evaluate the effect that the different choice points have on the performance of the resulting model.
Based on experimental results, one of the most important factors in achieving the best-performing tagger is the order in which the linearizer enumerates the tree and the alignment between the tags and the input sequence. 
We show that taggers with in-order linearization are the most accurate parsers, followed by their pre-order and post-order counterparts, respectively.\looseness=-1

Finally, we argue that the effect of the linearization step on parsing performance can be explained by the deviation between the resulting tag sequence and the input sequence, which should be minimized. We theoretically show that in-order linearizers always generate tag sequences with zero deviation. 
On the other hand, the deviation of pre- and post-order linearizations in the worst case grows in order of sentence length.
Empirically, experiments on a set of 9 typologically diverse languages show that the deviation varies across languages and negatively correlates with parsing performance.
Indeed, we find that the deviation appears to be the most important factor in predicting parser performance.\looseness=-1

\begin{table*}[t]
\centering
\begin{tabular}{@{}rlll@{}}\toprule
& \multicolumn{3}{c}{Linearizations} \\
\cmidrule(lr){2-4}
$\awtd$ & $\tdsrTagger(\dww)$ & $\busrTagger(\dww)$ & $\tetraTagger(\dww)$ \\ \midrule
${\scriptstyle \reduce(\SSS \rightarrow \PRP\, \VP)}$ & $\Circled{1}: \Leftchild(\SSS)$ & $\Circled{2}: \leftchild$ & $\Circled{2}: \leftchild$ \\ 
${\scriptstyle \shift(\PRP \rightarrow \textit{She})}$ & $\Circled{2}: \leftchild$ & $\Circled{4}: \leftchild$ & $\Circled{1}: \Leftchild(\SSS)$ \\
${\scriptstyle \reduce(\VP \rightarrow \V\, \VP)}$ & $\Circled{3}: \Rightchild(\VP)$ & $\Circled{6}: \leftchild$ & $\Circled{4}: \leftchild$ \\
${\scriptstyle \shift(\V \rightarrow \textit{enjoys})}$ & $\Circled{4}: \leftchild$ & $\Circled{7}: \rightchild$ & $\Circled{3}: \Rightchild(\VP)$ \\
${\scriptstyle \reduce(\VP \rightarrow \V\, \N)}$ & $\Circled{5}: \Rightchild(\VP)$ & $\Circled{5}: \Rightchild(\VP)$ & $\Circled{6}: \leftchild$ \\
${\scriptstyle \shift(\V \rightarrow \textit{reading})}$ & $\Circled{6}: \leftchild$ & $\Circled{3}: \Rightchild(\VP)$ & $\Circled{5}: \Rightchild(\VP)$ \\
${\scriptstyle \shift(\N \rightarrow \textit{papers})}$ & $\Circled{7}: \rightchild$ & $\Circled{1}: \Leftchild(\SSS)$ & $\Circled{7}: \rightchild$ \\
\bottomrule
\end{tabular}
\begin{tabular}{c}
Derivation Tree: $\dww$ \\ 
\begin{forest} 
 for tree={myleaf/.style={label={[align=center]below:{\strut#1}}}}
 [$\SSS^{\Circled{1}}$  [$\PRP^{\Circled{2}}$[\textit{She}] ] [$\VP^{\Circled{3}}$  [$\V^{\Circled{4}}$[\textit{enjoys}] ] [$\VP^{\Circled{5}}$ [$\V^{\Circled{6}}$[\textit{reading}] ] [$\N^{\Circled{7}}$[\textit{papers}] ] ]  ] ] 
\end{forest}
\end{tabular}
\caption{Examples of different linearizations for the derivation tree of the sentence \textit{She enjoys reading papers}.
Note that $\awtd$ is the action sequence of a shift--reduce parser, $\tdsrTagger(\dww)$ is the pre-order linearization, $\busrTagger(\dww)$ is the post-order linearization, and $\tetraTagger(\dww)$ is the in-order linearization (tetratagger).}
\label{tab:linearizer-examples}
\end{table*}

\section{Preliminaries}
Before diving into the details of parsing as tagging, we introduce the notation that is used throughout the paper. 
We assume that we have a weighted grammar in Chomsky normal form $\grammar = \langle \nonterm, \SSS, \alphabet, \rules \rangle$ where $\nonterm$ is a finite set of nonterminal symbols, $\SSS \in \nonterm$ is a distinguished start symbol, $\alphabet$ is an alphabet of terminal symbols, and $\rules$ is a set of productions. Each production can take on two forms, either $\X \rightarrow \Y\, \Z$, where $\X, \Y, \Z \in \nonterm$, or $\X \rightarrow x$, where $\X \in \nonterm$ and $x \in \Sigma$.\footnote{We assume that the empty string is not in the yield of $\grammar$---otherwise, we would require a rule of the form $\SSS \rightarrow \varepsilon$.}
Let $\ww = w_1 w_2 \cdots w_N$ be the input sequence of $N$ words, where $w_n \in \alphabet$. 
Next, we give a definition of a derivation tree $\dww$ that yields $\ww$.\looseness=-1 

\begin{defin}[Derivation Tree]\label{def:derivation}
A \defn{derivation tree} $\dww$ is a labeled and ordered tree. 
We denote the nodes in $\dww$ as $\dnode$.
The helper function $\rho(\cdot)$ returns the label of a node.
If the node is an interior node, then $\dlabel(\dnode) = \X$ where $\X \in \nonterm$.
If the node is a leaf, then  $\dlabel(\dnode) = w$ where $w \in \alphabet$.\looseness=-1
\end{defin}
\noindent 
Note that under \Cref{def:derivation} a derivation tree $\dww$ must have $|\ww| = N$ leaves, which, in left-to-right order, are labeled as $w_1, w_2, \dots, w_N$.
We additionally define the helper function $\pi$ that returns
the production associated with an interior node.
Specifically, if $\dnode$ is an interior node of $\dww$ in a CNF grammar, then $\production(\dnode)$ returns either $\rho(\dnode) \rightarrow \rho(\dnode')\, \rho(\dnode'')$ if $\dnode$ has two children $\dnode'$ and $\dnode''$ or $\rho(\dnode) \rightarrow \rho(\dnode')$ if $\dnode$ has one child $\dnode'$.
We further denote the non-negative weight of each production in the grammar with $\score \geq 0$.\footnote{Note that, for every $\ww$, we must have at least one derivation tree with a non-zero score in order to be able to normalize the distribution.
} 
We now define the following distribution over derivation trees:
\begin{equation} \label{eq:parse_prob}
    p(\dww \mid \ww) 
    \propto 
    \prod_{\dnode \in \dww} \score(\production(\dnode))
\end{equation}
where we abuse notation to use set inclusion $\dnode \in \dww$ to denote an interior node in the derivation tree $\dww$.
In words, \Cref{eq:parse_prob} says that the probability of a derivation tree is proportional to the product of the scores of its productions.

\section{Linearization} \label{sec:linearizalign}
In this and the following two sections, we go through the parsing-as-tagging pipeline and discuss the choice points at each step. 
We start with the first step, which is to design a linearizer.
The goal of this step is to efficiently encode all we need to know about the parse tree in a sequence of tags.

\begin{defin}[\citet{gomez-rodriguez-vilares-2018-constituent}]\label{defin:linearizer}
A \defn{linearizer} is a function $\tagger : \Dww \rightarrow \calT^M$ that maps a derivation tree $\dww$ to a sequence of tags of length $M$, where $\Dww$ is the set of derivation trees with yield $\ww$, $\calT$ is the set of tags, $M$ is $\bigo{N}$, and $N$ is the length of $\ww$.
We further require that $\tagger$ is a total and injective function.
This means that each derivation tree $\dww$ in $\Dww$ is mapped to a unique tagging sequence $\mathbf{t}$, but that some tagging sequences do not map back to derivation trees.\looseness=-1
\end{defin}

We wish to contextualize \Cref{defin:linearizer} 
by contrasting parsing-as-tagging schemes, as formalized in \Cref{defin:linearizer}, with classical transition-based parsers.
As written, \Cref{defin:linearizer} subsumes many transition-based parsers, e.g., \citeposs{nivre-2008-algorithms} arc-standard and arc-eager parsers.\footnote{\citet{nivre-2008-algorithms} discusses dependency parsers, but both arc-standard and arc-eager can be easily modified to work for constituency parsing, the subject of this work.}
However, in the case of most transition-based parsers, the linearizer $\tagger$ is a \emph{bijection} between the set of derivation trees $\Dww$ and the set of valid tagging sequences, which we denote as $\calT_{\ww} \subset \calT^M$.
The reason for this is that most transition-based parsers require a global constraint to ensure that the resulting tag sequence corresponds to a valid parse.\looseness=-1

In contrast, in \Cref{defin:linearizer}, we require a looser condition---namely, that $\tagger: \mathcal{D}_\ww \rightarrow \calT^M$ is an \emph{injection}.
Therefore, tagging-based parsers allow the prediction of invalid linearizations, i.e., sequences of tags that \emph{cannot} be mapped back to a valid derivation tree. 
More formally, 
the model can place non-zero score on elements in $\calT^M \setminus \calT_\ww$.
However, the hope is that the model \emph{learns} to place zero score on elements in $\calT^M \setminus \calT_\ww$, and, thus, enforces a hard constraint.
Therefore, parsing-as-tagging schemes come with the advantage that they make use of a simpler structure prediction framework, i.e., tagging.
Nevertheless, in practice, they still learn a useful parser that only predicts valid sequences of tags due to the expressive power of the neural network backbone.\looseness=-1

\subsection{Pre-Order and Post-Order Linearization}
In this section, we discuss how to define linearizers based on the actions that \defn{shift--reduce parsers} take.
A shift--reduce parser gives a particular linearization of a derivation tree via a sequence of \textsc{shift} and \textsc{reduce} actions.
A \emph{top-down} shift--reduce parser enumerates a derivation tree $\dww$ in a pre-order fashion, taking a $\shift(\X \rightarrow x)$ action when visiting a node $\dnode$ with $\production(\dnode) = \X \rightarrow x$, and a $\reduce(\X \rightarrow \Y\, \Z)$ action when $\production(\dnode) = \X \rightarrow \Y\, \Z$.
The parser halts when every node is visited.
As an example, consider the sentence: \textit{She enjoys reading papers}, with the parse tree depicted in \Cref{tab:linearizer-examples}.
The sequence of actions used by a top-down shift--reduce parser is shown under the column $\awtd$.
Given the assumption that the grammar is in Chomsky normal form, the parser takes $N$ \shift actions and $N-1$ \reduce actions, resulting in a sequence of length $M = 2N-1$. From this sequence of actions, we construct the tags output by the pre-order linearization $\tdsrTagger$ as follows: 
\begin{enumerate}
    \item We show reduce actions with double arrows $\Leftchild, \Rightchild$, and shift actions with $\leftchild, \rightchild$.
    \item For \shift actions, we encode whether the node being shifted is a left or a right child of its parent.\footnote{We assume that part-of-speech tags are given. 
    Otherwise, the identity of the non-terminal being shifted should be encoded in $\rightchild$ and $\leftchild$ tags as well.} We show left terminals with $\leftchild$ and right terminals with $\rightchild$.
    \item For \reduce actions, we encode the identity of the non-terminal being reduced as well as whether it is a left or a right child. 
    We denote reduce actions that create non-terminals that are left (right) children as $\Leftchild$ ($\Rightchild$).
\end{enumerate}
The output of such a linearization $\tdsrTagger$ is shown in \Cref{tab:linearizer-examples}. 
Similarly, a \emph{bottom-up} shift--reduce parser can be viewed as post-order traversal of the derivation tree; see $\busrTagger$ in \Cref{tab:linearizer-examples} for an example.

\subsection{In-Order Linearization}
On the other hand, the linearization function used in tetratagger does an in-order traversal \citet{liu-zhang-2017-order} of the derivation tree. 
Similarly to the pre- and post-order linearization, when visiting each node, it encodes the direction of the node relative to its parent in the tag sequence, i.e., whether this node is a left child or a right child.
For a given sentence $\ww$, the tetratagging linearization $\bt$ also has a length of $M=2N-1$ actions. 
An example can be found under the column $\tetraTagger$ in \Cref{tab:linearizer-examples}.
The novelty of tetratagging lies in how it builds a fixed-length tagging schema using an in-order traversal of the derivation tree. 
While this approach might seem intuitive, the connection between this particular linearization and shift--reduce actions is not straightforward. 
The next derivation states that the tetratag sequence can be derived by shift--reduce parsing on a \emph{transformed} derivation tree. \looseness=-1

\begin{der}[Informal] \label{thm:tetra-equivalence}
There exists a sequence transformation function that transforms bottom-up shift--reduce actions $\awbu$ on the \defn{right-corner transformed} derivation tree to tetratags. 
\end{der}

A similar statement holds for transforming top-down shift--reduce actions on the \defn{left-corner}\footnote{We refer the reader to \citet{johnson-1998-finite-state} for a close read on grammar transformations.} derivation trees to tetratags; a more formal derivation can be found in \cref{app:theorems}.\looseness=-1
\subsection{Deviation}
The output of the linearization function is a sequence of tags $\bt$. Next, we need to align this sequence with the input sequence $\ww$. 
Inspired by \citet{gomez-rodriguez-etal-2020-unifying}, we define an alignment function as follows.
\begin{defin}\label{defin:aligner}
The \defn{alignment function} $\aligner$ maps each word in an input sequence of length $N$ to at most $K$ tags in a tag sequence of length $M$: $\aligner(\ww, \bt) = t_{11}, t_{12},\dots, t_{1K},$ $\dots, t_{N1}, t_{N2}, \dots, t_{NK}$. We say $w_n$ is mapped to $t_{n1}, \dots, t_{nK}$.\looseness=-1
\end{defin}
Because all of the aforementioned linearizers generate tag sequences of length $2N-1$, a natural alignment schema, which we call \defn{paired alignment} $\paligner$, is to assign two tags to each word, except for the last word, which is only assigned to one tag. \cref{fig:pat-pipeline} depicts how such an alignment is applied to a sequence of pre-order linearizations with an example input sequence.
\begin{defin}\label{defin:deviation}
We define \defn{deviation} of $\aligner(\ww, \bt)$ to be the distance between a word and the tag corresponding to shifting this word into the stack. Formally, for the $n^{\text{th}}$ word in the sequence, if $t_{lk} = \shift(w_{n})$, then the deviation would be $|n - l|$. We further call a sequence of shift--reduce tags \defn{shift-aligned} with the input sequence if and only if it has zero deviation for any given word and sentence.\looseness=-1
\end{defin}
Note that if the paired alignment is used with an in-order shift--reduce linearizer, as done in tetratagger, it will be shift-aligned. However, such schema will not necessarily be shift aligned when used with pre- or post-order linearizations and leads to non-zero deviations. More formally, the following three propositions hold.
 
\begin{prop}
For any input sequence $\ww$, the paired alignment of this sequence with the in-order linearization of its derivation tree, denoted $\paligner(\ww, \tetraTagger)$, has a maximum deviation of zero, i.e., it is shift-aligned.
\end{prop}
\begin{proof}
In a binary derivation tree, in-order traversal results in a sequence of tags where each shift is followed by a reduce. Since words are shifted from left to right, each word will be mapped to its corresponding shift tag by $\paligner$, thus creating a shift-aligned sequence of tags.
\end{proof}
\begin{prop}
For any input sequence $\ww$ of length $N$, the maximum deviation of the paired alignment applied to a pre-order linearization $\paligner(\ww, \tdsrTagger)$ in the worst case is $\bigo{N}$.
\end{prop}
\begin{proof}
The deviation is bounded above by $\bigo{N}$.
Now consider a left-branching tree.
In this case, $\tdsrTagger$ starts with at most $N-1$ reduce tags before shifting the first word. Therefore, the maximum deviation, which happens for the first word, is $\lfloor \frac{N}{2} \rfloor$, which is $\bigo{N}$, achieving the upper bound.
\end{proof}
\begin{prop}
For any input sequence $\ww$ with length N, the maximum deviation of the paired alignment applied to a post-order linearization $\paligner(\ww, \busrTagger)$ in the worst case is $\bigo{N}$.
\end{prop}
\begin{proof}
The proof is similar to the pre-order case above, but with a right-branching tree instead of a left-branching one.
\end{proof}

Intuitively, if the deviation is small, it means that the structural information about each word is encoded close to the position of that word in the tag sequence.
Thus, it should be easier for the learner to model this information due to the locality.
Therefore, we expect a better performance from taggers with shift-aligned or low deviation linearizations. 
Later, in our experiments, we empirically test this hypothesis.\looseness=-1


\begin{figure*}[t]
    \centering
    \includegraphics[width=\linewidth]{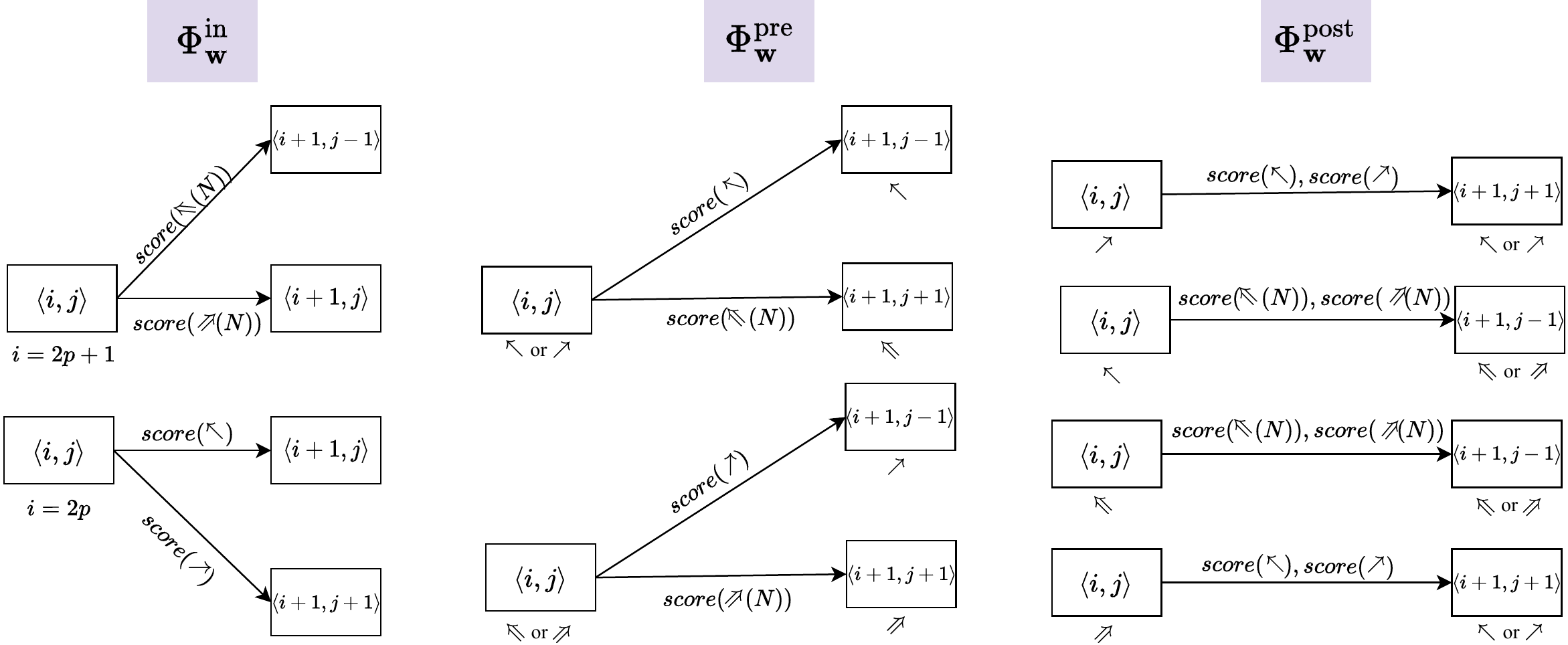}
    \caption{Decoding DAG templates for three linearization schemes.}
    \label{fig:decode}
\end{figure*}

\section{Learning} \label{sec:learn}
In this section, we focus on the second step in a parsing-as-tagging pipeline, which is to score a tag sequence. 
The parsers-as-taggers proposed in the literature often simplify the probability distribution over derivation trees, as defined in \cref{eq:parse_prob}.
We first discuss various approximations
we can make to the score function $\score(\X \rightarrow \Y \, \Z)$.
The three approximations we identify each make the model less expressive, but allow for more efficient decoding algorithms (see \cref{sec:decode}).
Then, we exhibit how the crudest approximation allows for a fully parallelized tagger.\looseness=-1 

\subsection{Approximating the Score Function}
We identify three natural levels of approximation for $\score(\X \rightarrow \Y \, \Z)$, which we explicate below.\looseness=-1

\paragraph{Independent Approximation.}
The first approximation is $\score(\X \rightarrow \Y \, \Z) \defeq \score(\bullet \rightarrow \Y\, \bullet) \times \score(\bullet \rightarrow \bullet \, \Z)$, which we call the \defn{independent} approximation.
We can enact the independent approximation under all three linearization schemes $\tdsrTagger$, $\busrTagger$, and $\tetraTagger$.
In each scheme, the tag encodes the nonterminal and its position with respect to its parent, i.e., whether it is a left child or a right child. The independent approximation entails that the score of each rule is the product of the scores of the left child and the right child. 

\paragraph{Left-Dependent Approximation.}
Second, we consider the approximation $\score(\X \rightarrow \Y \, \Z) \defeq \score(\X \rightarrow \Y \bullet)$, which we call the \defn{left-dependent} approximation.
The left-dependent approximation is only valid under the pre-order linearization scheme.
The idea is that, after generating a tag for a node $\dlabel(\dnode) = \X$ as a part of the production $\production(\dnode) = \X \rightarrow \Y \,\Z$, we immediately generate a tag for its left child $\Y$ \emph{independent} of $\Z$. 
Therefore, if we assume that the score of each tag only depends on the last generated tag, we simplify $\score(\X \rightarrow \Y \, \Z)$ as $\score(\X \rightarrow \Y \bullet)$.

\paragraph{Right-Dependent Approximation.}
The third approximation is the \defn{right-dependent} approximation, which takes the form $\score(\X \rightarrow \Y \, \Z) \defeq \score(\X \rightarrow \bullet \, \Z)$.
The right-dependent approximation is only valid under the post-order linearization scheme.
It works as follows: Immediately after generating the tag for the right child, the tag for its parent is generated. Again, if we assume that the score of each tag only depends on the last generated tag, we simplify $\score(\X \rightarrow \Y \, \Z)$ here as $\score(\X \rightarrow \bullet \, \Z)$.

\subsection{Increased Parallelism}
One of the main benefits of modeling parsing as tagging is that one is then able to fine-tune pre-trained models, such as BERT \citep{devlin-etal-2019-bert}, to predict tags efficiently \emph{in parallel} without the need to design a task-specific model architecture, e.g., as one does when building a standard transition-based parser.
In short, the learner tries to predict the correct tag sequence from pre-trained BERT word representations \emph{and} learns the constraints to enforce that the tagging encodes a derivation tree.

Given these representations, in the independent approximation, a common practice is to take $K$ \emph{independent} projection matrices and apply them to the last subword unit of each word, followed by a softmax to predict the distribution over the $k^{\text{th}}$ tag. 
Given a derivation tree $\dww$ for a sentence $\ww$, our goal
is to maximize the probability of the tag sequence that is the result of linearizing the derivation tree, i.e.,  $\tagger(\dww) = \tagging$.
The training objective is to maximize the probability
\begin{align}
    p(\tagging &\mid \ww)
    =\prod_{\substack{1 \leq n \leq N ,\\ 1 \leq k \leq K}} p(\rv{T}_{nk} = t_{nk} \mid \ww) \label{eq:bertprob} \\
    &= \prod_{\substack{1 \leq n \leq N, \\ 1 \leq k \leq K}} \mathrm{softmax}(W_k \cdot \mathrm{bert}(w_n \mid \ww))_{t_{nk}} \nonumber 
\end{align}
where $p(\rv{T}_{nk} = t_{nk} \mid \ww)$ is the probability of predicting the tag $t$ for the $k^{\text{th}}$ tag assigned to $w_n$, and $\mathrm{bert}(w_n \mid \ww)$ is the output of the last layer of $\mathrm{bert}$ processing $w_n$ in the context of $\ww$, which is a vector in $\mathbb{R}^d$. 
Note that \cref{eq:bertprob} follows from the independence approximation where the probability of each tag does \emph{not} depend on the other tags in the sequence. 
Due to this approximation, there exist distributions that can be expressed by \cref{eq:parse_prob}, but which  can \emph{not} be expressed by \cref{eq:bertprob}.

Making a connection between \cref{eq:parse_prob} and \cref{eq:bertprob} helps us understand the extent to which we have restricted the expressive power of the model through independent assumptions to speed up training. 
For instance, we can properly model tagging sequences under a left- and right-dependent approximation using a conditional random field \citep[CRF;][]{10.5555/645530.655813} with a first-order Markov assumption. 
Alternatively, we can create a dependency between the tags using an LSTM \citep{10.1162/neco.1997.9.8.1735} model. In our experiments, we compare these approximations by measuring the effect of the added expressive power on the overall parsing performance.\looseness=-1

\begin{figure*}[t]
    \centering
    \includegraphics[width=\linewidth]{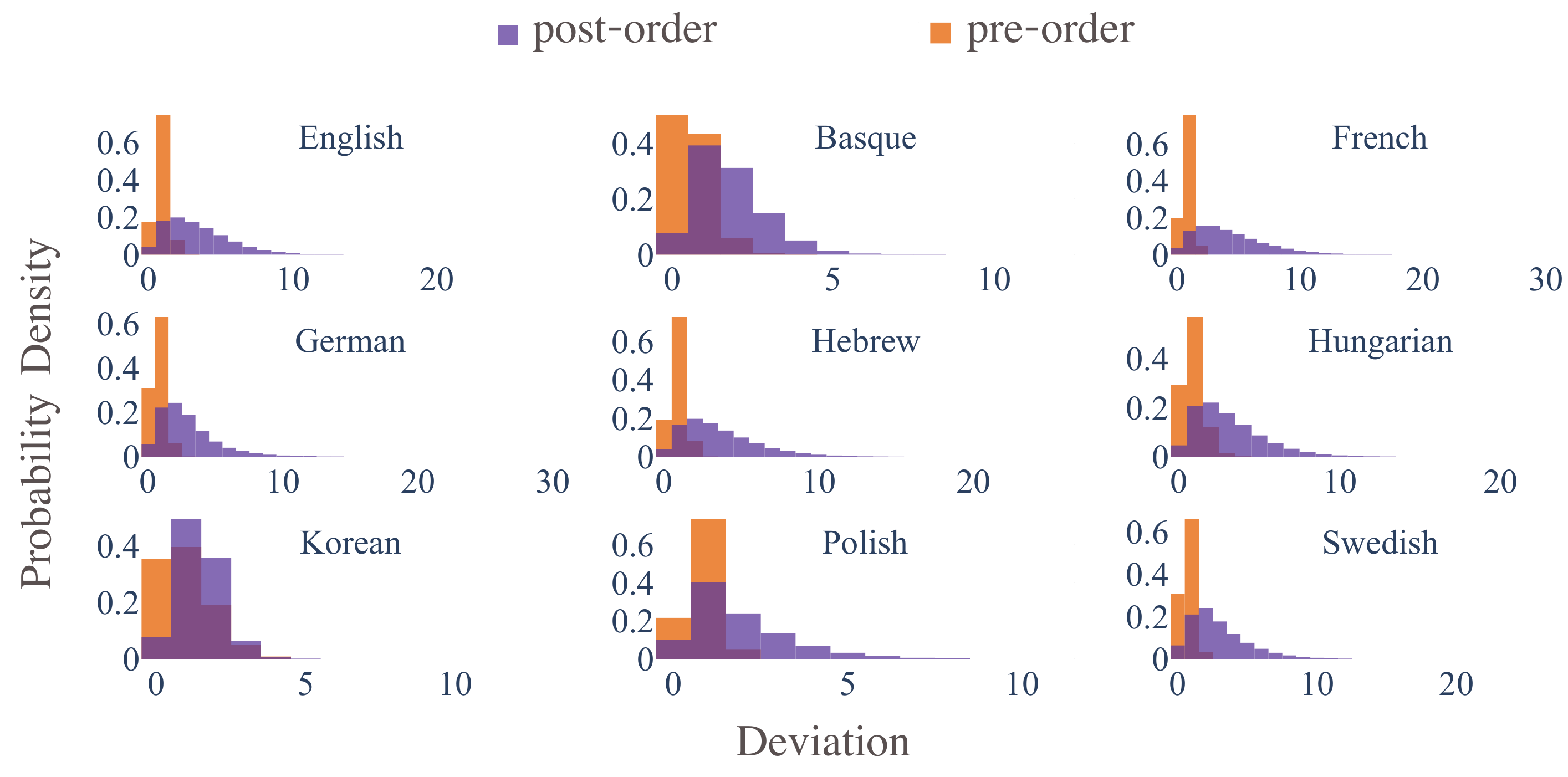}
    \caption{Distribution of word-level deviation in a random sample of 5000 sentences. In most languages, pre-order linearization gives a better alignment with the input sentence because the deviation peaks close to zero. However, the shape and the distance between the two distribution varies across languages.}
    \label{fig:deviation}
\end{figure*}

\section{Decoding} \label{sec:decode}
In this section, we focus on the last step in a parsing-as-tagging pipeline, which is to decode a tag sequence into a tree.
The goal of the decoding step is to find a \emph{valid} sequence of tags $\bt^*$ that are assigned the highest probability under the model for a given sentence, i.e.,
\begin{align}
    \bt^* \defeq \argmax_{\bt \in \mathcal{T}_{\ww}}\, p(\bt \mid \ww)
\end{align}
where $\mathcal{T}_{\ww}$ is the set of tag sequences with a yield of $\ww$. 
Again, we emphasize that not all sequences of tags are valid, i.e., not all tagging sequences can be mapped to a derivation tree, 
Therefore, in order to ensure we always return a valid tree, the invalid sequences must be detected and weeded out during the decoding process.
This will generally require a more complicated algorithm.\looseness=-1

\subsection{Dynamic Programming}
We extend the dynamic program (DP) suggested by \citet{kitaev-klein-2020-tetra} for decoding an in-order linearization to other variations of linearizers.
\citeposs{kitaev-klein-2020-tetra} dynamic program relies on the fact that the validity of a tag sequence does \emph{not} depend on individual elements in the stack.
Instead, it depends only on the \emph{size} of the stack at each derivation point.
We show that the same observation holds for both $\busrTagger$ and $\tdsrTagger$, which helps us to develop an efficient decoding algorithm for these linearizations.\looseness=-1
\paragraph{Decoding $\tetraTagger$.}
We start by introducing \citeposs{kitaev-klein-2020-tetra} dynamic program for their $\tetraTagger$ model.
Like many dynamic programs, theirs can be visualized as finding the highest-scoring path in a weighted directed acyclic graph (DAG). 
Each node in the DAG represents the number of generated tags $i$ and the current stack size $j$. 
Each edge represents a transition weighted by the score of the associated tag predicted by the learner. Only sequences of exactly $2N-1$ tags will be valid.
The odd tags must either be $\leftchild$ or $\rightchild$,\footnote{With the exception of the first tag that can only be $\leftchild$.} and the even tags must be either $\Leftchild$ or $\Rightchild$. The only accepted edges when generating the $i^{\text{th}}$ tag ($1 < i \leq 2N-1$) in the sequence are shown in \cref{fig:decode} left. The starting node is $\langle1,1\rangle$, with $\leftchild$ in the stack, and the goal is to reach the node: $\langle2N-1,1\rangle$. 

\paragraph{Decoding \normalfont{$\tdsrTagger$} \textbf{and} \normalfont{$\busrTagger$\textbf{.}}} 
To adapt the algorithm introduced above to work with pre-order linearization, we first investigate what makes a pre-order tagging sequence valid:
\begin{itemize}
    \item When a $\reduce$ tag ($\Leftchild$ or $\Rightchild$) is generated, there must be at least one node in the stack, and after performing the $\reduce$, the size of the stack is increased by one.
    \item Immediately after shifting a node, we either process (perform a $\shift$ or $\reduce$ on) that node's right sibling or shift the right sibling. Therefore, the only valid tags to generate subsequently are $\rightchild$ or $\Rightchild$.
    \item Immediately after reducing a node, we process (perform a $\shift$ or $\reduce$ on) its left child. Therefore, it is only valid to generate either $\leftchild$ or $\Leftchild$.
\end{itemize}

The above constraints suggest that if we know the previous tag type (\textsc{shift} or \textsc{reduce}) and the stack size at each step of generation, we can form the valid transitions from each node $\langle i, j\rangle$. Such transitions are shown in \cref{fig:decode} center. Similarly, the valid transitions for decoding post-order linearization are shown in the right part of \cref{fig:decode}.

\paragraph{Time Complexity.} The complexity of the dynamic program to find the highest-scoring path depends on the number of unique tags $\bigo{|\nonterm|}$, the length of the tag sequence $\bigo{N}$, and the stack size $d$.
Linearization schemes derived from the right-corner transform tend to have smaller stack sizes \citep{abney1991memory, schuler-etal-2010-broad}.\footnote{For purely left- or right-branching trees, the left-corner or right-corner transformations, respectively, provably reduce the stack size to one \citep{johnson-1998-finite-state}. 
These transformations could increase the stack size for centered-embedded trees. However, such trees are rare in natural language as they are difficult for humans to comprehend \citep{GIBSON19981}.}
As a result, one can often derive a faster, \emph{exact} decoding algorithm for such schemes. Please refer to \Cref{app:dps} for a more formal treatment of the dynamic program for finding the highest-scoring tag sequence using these transitions (\Cref{alg:dp-indep}).

\paragraph{An $\bigo{dN|\nonterm|^2}$ Algorithm.}
As discussed in \cref{sec:learn}, one can break the independence assumption in \cref{eq:bertprob} by making each tag dependent on the previously generated tag. Consequently, when finding the highest-scoring path in the decoding step, we should remember the last generated tag.
Therefore, in addition to $i$ and $j$, the generated tag must be stored.
This adds $\bigo{|\nonterm|}$ extra memory per node, and increases the runtime complexity to $\bigo{dN|\nonterm|^2}$. 
Please refer to \Cref{app:dps} for a more formal treatment of the dynamic program for finding the highest-scoring tag sequence (\Cref{alg:dp-dep}).

\subsection{Beam Search Decoding}
We can speed up the dynamic programming algorithms above if we forego the desire for an exact algorithm.
In such a case, applying beam search presents a reasonable alternative where  we only keep track of $h$ tag sequences with the highest score at each step of the tag generation.
Although no longer an exact algorithm, beam search reduces the time complexity to $\bigo{h \log h\, N|\nonterm|}$.\footnote{This can be improved to $\bigo{h N|\nonterm|}$ with quick select.}
This can result in a substantial speed-up over the dynamic programs presented above when $d$, the maximum required stack size, is much larger than the beam size $h$.
We empirically measure the trade-off between accuracy and speed in our experiments.

\begin{table*}[t] 
\centering 
\tabcolsep=0.11cm
\begin{tabular}{@{}lcccccccc@{}}\toprule
& Basque & French & German & Hebrew & Hungarian & Korean & Polish & Swedish \\ \midrule
\citet{kitaev-etal-2019-multilingual} & 91.63 & 87.42 & 90.20 & 92.99 & 94.90 & 88.80 & 96.36 & 88.86 \\
\citet{kitaev-klein-2018-constituency} & 89.71 & 84.06 & 87.69 & 90.35 & 92.69 & 86.59 & 93.69 & 84.45 \\ \midrule
\texttt{in-order}-\textsc{[indep.]} & \textbf{89.86} & 84.54 & \textbf{88.34} & \textbf{91.40} & \textbf{93.81} & 84.89 & \textbf{95.20} & \textbf{86.66} \\ 
\texttt{pre-order}-\textsc{[indep.]} & 87.98 & \textbf{84.76} & 88.18 & 89.45 & 90.69 & 81.81 & 94.84 & 84.65 \\
\texttt{post-order}-\textsc{[indep.]} & 81.05 & 56.97 & 78.07 & 64.21 & 76.24 & \textbf{86.64} & 86.91 & 58.44 \\
\bottomrule
\end{tabular}
\caption{Comparison of FMeasure of different tagging schemata on the SPMRL test set}
\label{tab:parsing-metrics-multi}
\end{table*}

\section{Experiments}
In our experimental section we aim to answer two questions: (i) What is the effect of each design decision on the efficiency and accuracy of a tagger, and which design decisions have the greatest effect? (ii) Are the best design choices consistent across languages?\looseness=-1

\paragraph{Data.}
We use two data sources: the Penn Treebank \citep{marcus-etal-1993-building} for English constituency parsing and Statistical Parsing of Morphologically Rich Languages (SPMRL) 2013/2014 shared tasks \citep{seddah-etal-2013-overview, seddah-etal-2014-introducing} for 8 languages: Basque, French, German, Hebrew, Hungarian, Korean, Polish, and Swedish.
We provide the dataset statistics in \Cref{tab:dataset}.
We perform similar preprocessing as \citet{kitaev-klein-2020-tetra}, which we explain in \cref{app:preprocess}.
For evaluation, the \textsc{evalb} Perl script\footnote{\href{http://nlp.cs.nyu.edu/evalb/}{\texttt{http://nlp.cs.nyu.edu/evalb/}}} is used to calculate FMeasure of the parse tree.
We use only \emph{one} GPU node for the reported inference times.
Further experimental details can be found in \cref{app:exp}.\looseness=-1

\subsection{What Makes a Good Tagger?}
\paragraph{Linearization.} First, we measure the effect of linearization and alignment on the performance of the parser. We train BERT\footnote{We use \textsc{bert-large-uncased} from the Huggingface library \citep{wolf-etal-2020-transformers}.} on the English dataset.
While keeping the learning schema fixed, we experiment with three linearizations: $\tetraTagger$, $\busrTagger$, and $\tdsrTagger$. 
We observe that the in-order shift--reduce linearizer leads to the best FMeasure ($95.15$), followed by the top-down linearizer ($94.56$). 
There is a considerable gap between the performance of a tagger with the post-order linearizer ($89.23$) and other taggers, as shown in \Cref{tab:parsing-metrics}. 
The only difference between these taggers is the linearization function, and more specifically, the deviation between tags and the input sequence as defined in \cref{defin:deviation}.
This result suggests that deviation is an important factor in designing accurate taggers.


\paragraph{Learning Schema.} Second, we measure the effect of the independence assumption in scoring tagging sequences. We slightly change the scoring mechanism by training a CRF\footnote{We use our custom implementation of CRF with pytorch.} and an LSTM layer\footnote{We use a two-layered biLSTM network, with the hidden size equal to the tags vocabulary size $|\mathcal{T}|$ (approximately 150 cells). For further details please refer to \cref{app:exp}.} to make scoring over tags left- or right-dependent.
We expect to see improvements with $\tdsrTagger$ and $\busrTagger$, since in these two cases, creating such dependencies gives us more expressive power in terms of the distributions over trees that can be modeled.
However, we only observe marginal improvements when we add left and right dependencies (see \Cref{tab:parsing-metrics}). 
Thus, we hypothesize that BERT already captures the dependencies between the words.\looseness=-1

\begin{table}
\centering
\resizebox{\columnwidth}{!}{%
\begin{tabular}{@{}lcccc@{}}\toprule
& \multicolumn{2}{c}{Beam Search} & \multicolumn{2}{c}{DP} \\
\cmidrule(lr){2-3}
\cmidrule(lr){4-5}
& FMeasure & Sents/s & FMeasure & Sents/s \\ \midrule
\makecell{\texttt{in-order}$^\star$ \\ \textsc{[indep.]}} & 91.59 & 156 & 95.15 & 128 \\ \midrule
\makecell{\texttt{pre-order} \\ \textsc{[indep.]}} & 93.57 & 114 & 94.56 &  51 \\
\makecell{\texttt{pre-order} \\ \textsc{[left dep. crf]}} & 93.65 & 110 & 94.47 & 21 \\
\makecell{\texttt{pre-order} \\ \textsc{[left dep. lstm]}} & 93.78 & 110 & 94.38 & 20 \\ \midrule
\makecell{\texttt{post-order} \\ \textsc{[indep.]}} & 85.38 & 108  & 89.23 & 29\\
\makecell{\texttt{post-order} \\ \textsc{[right dep. crf]}} & 82.28 & 108  & 88.28 & 5 \\
\makecell{\texttt{post-order} \\ \textsc{[right dep. lstm]}} & 85.25 & 103 & 89.61 & 5\\
\bottomrule
\end{tabular}}
\caption{Comparison of parsing metrics on the WSJ test set. 
$^\star$This is the exact equivalent setup to tetratagger. However, because we neither use the exact same code nor the hardware, the numbers do not exactly match with what is reported in the original paper.}
\label{tab:parsing-metrics}
\end{table}

\paragraph{Decoding Schema.} 
Third, we assess the trade-off between accuracy and efficiency caused by using beam search versus exact decoding. 
We take $h=10$ as the beam size. As shown in \Cref{tab:parsing-metrics}, taggers with post-order linearization take longer to decode using the exact dynamic program.
This is largely due to the fact that they need deeper stacks in the decoding process (see \cref{app:stack} for a comparison).
In such scenarios, we see that beam search leads to a significant speed-up in decoding time, at the expense of a drop (between 3 to 6 points) in FMeasure.
However, the drop in accuracy observed with the pre-order linearization is less than 1 point.
Therefore, in some cases, beam search may be able to offer a sweet spot between speed and accuracy.
To put these results into context, we see taggers with in-order and pre-order linearizations achieve comparable results to state-of-the-art parsers with custom architectures, where the state-of-the-art parser achieves 95.84 FMeasure \citep{zhou-zhao-2019-head}.\looseness=-1

\subsection{Multilingual Parsing}
As shown in the previous experiment, linearization plays an important role in constructing an accurate tagger.
Moreover, we observe that in-order linearizers yield the best-performing taggers, followed by the pre-order and post-order linearizers. In our second set of experiments, we attempt to replicate this result in languages other than English.
We train BERT-based learners\footnote{We use \textsc{bert-base-multilingual-cased}.} on 8 additional typologically diverse languages.
As the results in \Cref{tab:parsing-metrics-multi} suggest, in general, taggers can achieve competitive FMeasures relative to parsers with custom architectures \citep{kitaev-klein-2018-constituency}. 

Similar to results in English, in most of the other eight languages, taggers with in-order linearizers achieve the best FMeasure.
We note that the gap between the performance of post and pre-order linearization varies significantly across languages.
To further investigate this finding, we compute the distribution of deviation in alignment at the word level (based on \cref{defin:deviation}). 
For most languages, the deviation of the pre-order linearizer peaks close to zero.
On the contrary, in French, Swedish, and Hebrew, the post-order linearizers do not align very well with the input sequence.
Indeed, we observe derivations of up to 30 for some words in these languages.
This is not surprising because these languages tend to be right-branching.
In fact, we observe a large gap between the performance of taggers with pre- and post-order linearizers in these three languages.\looseness=-1

On the other hand, for a left-branching language like Korean, post-order linearization aligns nicely with the input sequence.
This is also very well reflected in the parsing results, where the tagger with post-order linearization performs better than taggers using in-order or pre-order linearization. 

Our findings suggest that a tagger's performance depends on the deviation of its linearizer. We measure this via Pearson's correlation between parsers' FMeasure and the mean deviation of the linearizer, per language. We see that the two variables are negatively correlated ($-0.77$) with a $p$ value of $0.0001$.
We conclude that the deviation in the alignment of tags with the input sequence is highly dependent on the characteristics of the language.\looseness=-1

\section{Conclusion}
In this paper, we analyze parsing as tagging, a relatively new paradigm for developing statistical parsers of natural language. 
We show many parsing-as-tagging schemes are actually closely related to the well-studied shift–-reduce parsers. 
We also show that \citeposs{kitaev-klein-2020-tetra} tetratagging scheme is, in fact, an in-order shift--reduce parser, which can be derived from bottom-up or top-down shift--reduce parsing on the right or left-cornered grammars, respectively. 
We further identify three common steps in parsing-as-tagging pipelines and explore various design decisions at each step. 
Empirically, we evaluate the effect of such design decisions on the speed and accuracy of the resulting parser.
Our results suggest that there is a strong negative correlation between the deviation metric and the accuracy of constituency parsers.\looseness=-1 

\section*{Ethical Concerns}
We do not believe the work presented here further amplifies biases already present in the datasets and the algorithms that we experiment with because our work is primarily a theoretical analysis of existing work.
Therefore, we foresee no ethical concerns in this work.\looseness=-1

\section*{Limitations}
This work only focuses on constituency parsing. 
Therefore, the results might not fully generalize to other parsing tasks, e.g., dependency parsing. Furthermore, in our experiments, we only focus on comparing design decisions of parsing-as-tagging schemes, and not on achieving state-of-the-art results. 
We believe that by using larger pre-trained models, one might be able to obtain better parsing performance with the tagging pipelines discussed in this paper. In addition, in our multilingual analysis, the only left-branching language that we had access to was Korean, therefore, more analysis needs to be done on left-branching languages.

\section*{Acknowledgments}
We would like to thank Nikita Kitaev, Tim Vieira, Carlos Gómez-Rodríguez, and Jason Eisner for fruitful discussions on a preliminary draft of the paper. We would also like to thank anonymous reviewers for their insightful feedback during the review process. We also thank Zeerak Talat and Clément Guerner for their feedback on this manuscript. Afra Amini is supported by ETH AI Center doctoral fellowship. 

\bibliographystyle{acl_natbib}
\bibliography{acl2020}
\clearpage
\newpage
\appendix
\onecolumn

\section{Related Work} \label{app:literature}
We provide a comparative review of parsing as tagging works on both dependency and constituency parsing, focusing on their design decisions at different steps of the pipeline. 
\subsection{Dependency Parsing as Tagging}
Previous work on dependency parsing as tagging linearizes the dependency tree by iterating through the dependency arcs and encoding the relative position of the child with respect to its parent in the tag sequence \citep{li-etal-2018-seq2seq, kiperwasser-ballesteros-2018-scheduled}. 
Each word is then aligned with the tag indicating its relative position with respect to its head. At the decoding step, to prevent the model from generating dependency arcs that do not form valid dependency trees, a tree constraint is often applied \citep{li-etal-2018-seq2seq}. Similarly, \citet{vacareanu-etal-2020-parsing} employ contextualized embeddings of the input with the same tagging schema. 
\citet{strzyz-etal-2019-viable} provide a framework and train and compare various dependency parsers as taggers. See also \citet{strzyz2021viability} for further comparison.

Closest to our work, \citet{gomez-rodriguez-etal-2020-unifying} suggest a linearization and alignment scheme that works with \emph{any} transition-based parser to create a sequence of tags from a sequence of actions.
Therefore, their method applies to both constituency and dependency parsing. 
For instance, turning a shift--reduce parser into a tag sequence requires creating a tag for all the \reduce actions up to each \shift action. 
While this construction results in exactly $N$ tags, there are in general an exponentially large number of tags. 
To compare this approach with the linearizations discussed in this paper, we must note that while such a schema is shift-aligned, it might not \emph{evenly} align the tags with the input sequence.
Moreover, the in-order linearization scheme discussed here has a small number of tags.
\subsection{Constituency Parsing as Tagging}
\citet{gomez-rodriguez-vilares-2018-constituent} provide an in-order linearization of the derivation tree, where they encode nonterminals and their depth as a tagging sequence.
Since the depth is encoded in the tags themselves, the tag set size is unbounded and is dependent on the input. 
\citet{kitaev-klein-2020-tetra} improve upon \citeposs{gomez-rodriguez-vilares-2018-constituent} encoding by discovering a similar tagging system that has only four tags. 
In \citeposs{kitaev-klein-2020-tetra} linearization step, they encode both terminals and nonterminals in the tag sequence, as well as the direction of a node with respect to its parent. 
Using a pretrained BERT model to score tags, they show that their approach reaches state-of-the-art performance despite being architecturally simpler.
Furthermore, the parser is significantly faster than competing approaches.\looseness=-1

\begin{table*}[t] 
\renewcommand{\arraystretch}{0.7}
\begin{tabular}{@{} p{1.2cm} p{2.5cm} p{3cm} p{2.5cm} p{2.8cm} p{2.5cm} @{}}
$\dww$ & 
\footnotesize \begin{forest} 
 for tree={myleaf/.style={label={[align=center]below:{\strut#1}}}}
 [$\N$[$\N_1$[$\X_i$, myleaf={\textcolor{gray}{$w_i$}}, circle, draw, fill=lightgray][,phantom]][,phantom]]
 \end{forest} & \footnotesize \begin{forest}
 for tree={myleaf/.style={label={[align=center]below:{\strut#1}}}}
 [$\N$[,phantom][$\N_1$ [$\X_i$, myleaf={\textcolor{gray}{$w_i$}}, circle, draw, fill=lightgray][,phantom]]]
 \end{forest} & \footnotesize \begin{forest}
 for tree={myleaf/.style={label={[align=center]below:{\strut#1}}}}
 [[$\N$ [$\N_1$ [,phantom] [$\X_i$, myleaf={\textcolor{gray}{$w_i$}}, circle, draw, fill=lightgray]] [,phantom]][,phantom]]
 \end{forest} & \footnotesize \begin{forest}
 for tree={myleaf/.style={label={[align=center]below:{\strut#1}}}}
 [[,phantom][$\N$ [$\N_1$ [,phantom] [$\X_i$, myleaf={\textcolor{gray}{$w_i$}}, circle, draw, fill=lightgray]] [,phantom]]] 
 \end{forest} & \footnotesize \begin{forest}
 for tree={myleaf/.style={label={[align=center]below:{\strut#1}}}}
 [$\SSS$ [,phantom] [\dots [,phantom] [$\N_1$ [,phantom] [$\X_i$, myleaf={\textcolor{gray}{$w_i$}}, circle, draw, fill=lightgray]]]] \end{forest} \\
 $\tetraTagger$ & 
 $\leftchild \Leftchild{\scriptstyle (\N_1)}$ & $\leftchild \Rightchild{\scriptstyle (\N_1)}$ & $\rightchild \Leftchild{\scriptstyle (\N)}$ &  $\rightchild \Rightchild{\scriptstyle (\N)}$ & $\rightchild$ \\ \midrule
 $\dwrc$ & 
 \footnotesize \begin{forest} 
 for tree={myleaf/.style={label={[align=center]below:{\strut#1}}}}
 [$\N_1/.$[$\N_1/.$[$\varepsilon$][$\X_i$, myleaf={\textcolor{gray}{$w_i$}}, circle, draw, fill=lightgray]][,phantom]]
 \end{forest} & \footnotesize \begin{forest}
 for tree={myleaf/.style={label={[align=center]below:{\strut#1}}}}
 [$\N'$[$\N'/.$ [$\N'/\N_1$] [$\X_i$, myleaf={\textcolor{gray}{$w_i$}}, circle, draw, fill=lightgray]][,phantom]]
 \end{forest} & \footnotesize \begin{forest}
 for tree={myleaf/.style={label={[align=center]below:{\strut#1}}}}
 [[$\N/.$ [$\varepsilon$][$\N_1$ [$\N_1/.$] [$\X_i$, myleaf={\textcolor{gray}{$w_i$}}, circle, draw, fill=lightgray]]][,phantom]]
 \end{forest} & \footnotesize \begin{forest}
 for tree={myleaf/.style={label={[align=center]below:{\strut#1}}}}
 [[$\N'/.$ [$\N'/\N$][$\N_1$ [$\N_1/.$] [$\X_i$, myleaf={\textcolor{gray}{$w_i$}}, circle, draw, fill=lightgray]]][,phantom]]
  \end{forest} & \footnotesize \begin{forest}
 for tree={myleaf/.style={label={[align=center]below:{\strut#1}}}}
 [$\SSS$ [$\SSS/.$] [$\X_i$, myleaf={\textcolor{gray}{$w_i$}}, circle, draw, fill=lightgray]]]
  \end{forest}\\
  $\busrTagger$ & 
  \makecell[l]{${\scriptstyle \rightchild}$ \\ ${\scriptstyle \Leftchild(\N_1/. \rightarrow \varepsilon \, \X_i)}$} & \makecell[l]{${\scriptstyle \rightchild}$ \\  ${\scriptstyle \Leftchild(\N'/. \rightarrow \N'/\N_1 \, \X_i)}$} & \makecell[l]{${\scriptstyle \rightchild}$ \\ ${\scriptstyle \Rightchild(\N_1 \rightarrow \N_1/. \, \X_i)}$ \\ ${\scriptstyle \Leftchild(\N/. \rightarrow \varepsilon \, \N_1)}$} & 
  \makecell[l]{${\scriptstyle \rightchild}$ \\ ${\scriptstyle \Rightchild(\N_1 \rightarrow \N_1/. \, \X_i)}$ \\ ${\scriptstyle \Leftchild(\N'/. \rightarrow \N'/\N \, \N_1)}$} &
  \makecell[l]{${\scriptstyle \rightchild}$ \\ ${\scriptstyle \Leftchild(\SSS \rightarrow \SSS/. \, \X_i)}$} \\ [1cm]
  $\maprc$ & 
  $\leftchild \Leftchild{\scriptstyle (\N_1)}$ & $\leftchild \Rightchild{\scriptstyle (\N_1)}$ & $\leftchild \rightchild \Leftchild{\scriptstyle (\N)}$ & $\leftchild \rightchild \Rightchild{\scriptstyle (\N)}$ & $\leftchild \rightchild$ \\
  $\mergerc$ &
  $ \leftchild \Leftchild{\scriptstyle (\N_1)}$ & $\leftchild \Rightchild{\scriptstyle (\N_1)}$ & $\rightchild \Leftchild{\scriptstyle (\N)}$ & 
  $ \leftchild \Rightchild{\scriptstyle (\N)}$ & $\rightchild$
\end{tabular}
\caption{Equivalence of tetratags and bottom-up shift--reduce tags on the right-cornered derivation tree}
\label{tab:tetra-rc}
\end{table*}
\section{Derivations} \label{app:theorems}
The following derivation provides the exact map and merge functions for transforming shift--reduce action sequences to tetratags. 
\begin{der}
Let $\dwrc$ be the derivation tree  after the right-corner transformation for the sentence $\ww$, and let $\awrc$ be the sequence of \shift or \reduce actions taken by a bottom-up shift--reduce parser processing $\dwrc$. 
Then, we have that
\begin{equation} 
    \tetraTagger(\dww) = \mergerc \Big(\maprc\left(\awrc \right)\Big) 
\end{equation}
where $\mathrm{map}$ applies to each element in the action sequence and maps each action to:

\begin{align*}
    & \maprc(\reduce(\N \rightarrow \N/\N_1\; \N_2)) = \rightchild \\
    & \maprc(\reduce(\N/\N_1 \rightarrow \N/\N_2\; \N_3)) = \Rightchild(\N_2)\\
    & \maprc(\reduce(\N/\N_1 \rightarrow \varepsilon\; \N_2)) = \Leftchild(\N)\\
    & \maprc(\shift(\bullet)) = \leftchild
\end{align*}
The $\mergerc$ function reads the sequence of tags from left to right and whenever it encounters a $\leftchild$ tag followed immediately by a $\rightchild$, it merges them to a $\rightchild$ tag.
\end{der}
\begin{proof}
We split tetratags $\tetraTagger(\dww)$ into $n$ groups, where the first $n-1$ groups consist of two tags and the last group consists of only one tag. Let's focus on a specific split. According to the definition of tetratagger, this split starts with either $\leftchild$ or $\rightchild$, each of which corresponds to a terminal node, i.e. a word in the input sequence: $w_i$ with part-of-speech $\X_i$. Depending on the topology of $\dww$, there are five distinct ways to position $w_i$ in the derivation tree, each of which corresponds to unique tetratags for the split. These configurations are shown in the first row of \Cref{tab:tetra-rc}. Now we apply the right-corner transformation to the derivation trees and we obtain the trees shown in the second row of \Cref{tab:tetra-rc}. Next, we generate bottom-up shift--reduce linearizations for the transformed trees and split them at \shift actions. We see that after applying $\maprc$ and $\mergerc$ on each of the splits, we obtain the exact same tetratags for all of the five possible configurations, as shown in the last row of \Cref{tab:tetra-rc}.
\end{proof}

\begin{der} \label{app:lc}
Let $\dwlc$ be the \textbf{left-corner} transformed derivation tree for the sentence $\ww$, and $\awlc$ be the sequence of \shift or \reduce actions took by a top-down shift--reduce parser processing $\dwlc$. Then: 
\begin{equation} 
    \tetraTagger(\dww) = \mergelc \Big(\maplc\left(\awlc \right)\Big)
\end{equation}

where $\mathrm{map}$ applies to each element in the action sequence and maps each action to:
\begin{align*}
    & \maprc(\reduce(\N \rightarrow \N_2\; \N/\N_1)) = \leftchild \\
    & \maprc(\reduce(\N/\N_1 \rightarrow \N_3\; \N/\N_2)) = \Leftchild(\N_2)\\
    & \maprc(\reduce(\N/\N_1 \rightarrow \N_2\; \varepsilon)) = \Rightchild(\N)\\
    & \maprc(\shift(\bullet)) = \rightchild
\end{align*}
And $\mergerc$ function reads the sequence of tags from left to right and whenever it encounters a $\leftchild$ tag followed immediately by a $\rightchild$, it merges them to a $\leftchild$ tag.
\end{der}

\begin{proof}
The proof is similar to the proof of right-corner and bottom-up shift--reduce. 
\end{proof}

\section{Decoding Algorithms} \label{app:dps}
\begin{prop}
Given a scoring function with independent approximation, $\score(\X \rightarrow \Y \, \Z) \defeq \score(\bullet \rightarrow \Y\, \bullet) \times \score(\bullet \rightarrow \bullet \, \Z)$, over derivation trees, we can find the highest-probability derivation in $\bigo{dN|\nonterm|}$ using \Cref{alg:dp-indep}.
\end{prop}
\begin{proofsketch}
We give an outline of the construction.
First, consider a grammar in Chomsky normal form $\grammar$.
Next, apply the standard shift--reduce transformation to $\grammar$ in order to construct a pushdown automaton (PDA). 
Finally, note that under the independence assumption $\score(\X \rightarrow \Y \, \Z) \defeq \score(\bullet \rightarrow \Y\, \bullet) \times \score(\bullet \rightarrow \bullet \, \Z)$ we can reduce the runtime of the dynamic program that sums over all runs in the PDA from $\bigo{N^3 |\nonterm|^3}$ \citep{lang, alexandra} to $\bigo{d N |\nonterm|}$ because we only need to keep track of the \emph{height} of the stack $d$ and not which elements are in it. Note that $d = N$ in the worst case.
This can be viewed as using \Cref{alg:dp-indep} to find the valid tagging sequence with the highest score in a directed acyclic graph, as shown in \Cref{fig:decode}.
\end{proofsketch}
\begin{prop}
Suppose we are given a scoring function with left- or right-dependent approximation, $\score(\X \rightarrow \Y \, \Z) \defeq \score(\X \rightarrow \Y \bullet)$ or $\score(\X \rightarrow \Y \, \Z) \defeq \score(\X \rightarrow \bullet \, \Z)$, over derivation trees.
We can find the derivation tree with the highest score in $\bigo{dN|\nonterm|^2}$ using \Cref{alg:dp-dep}.
\end{prop}
\begin{proofsketch}
We give an outline of the construction.
First, consider a grammar in Chomsky normal form $\grammar$.
Next, apply the standard shift--reduce transformation to $\grammar$ in order to construct a pushdown automaton (PDA). 
Finally, note that under the left- or right-dependent approximation, $\score(\X \rightarrow \Y \, \Z) \defeq \score(\X \rightarrow \Y \bullet)$ or $\score(\X \rightarrow \Y \, \Z) \defeq \score(\X \rightarrow \bullet \, \Z)$, we can reduce the runtime of the dynamic program that sums over all runs in the PDA from $\bigo{N^3 |\nonterm|^3}$ to $\bigo{d N |\nonterm|^2}$ because we only need to keep track of the \emph{height} of the stack $d$, and the top element on the stack, but not the remaining elements in the stack.
This can be viewed as using \Cref{alg:dp-dep} to find the valid tagging sequence with the highest score in a directed acyclic graph, as shown in \Cref{fig:decode}.
\end{proofsketch}
\begin{algorithm*}
\caption{Dynamic program for decoding with independent scores}
\label{alg:dp-indep}
\begin{algorithmic}[1]
\footnotesize
\Procedure{DP}{$\mathrm{score}$}
\State $W = \mathbf{0}$ \Comment{initialize the DP chart to zero}
\State $W[1, 1] = \mathrm{log}\,\mathrm{score}(\leftchild)$ \Comment{pre-order: $W[1, 2] = \mathrm{log}\,\mathrm{score}(\Leftchild(\SSS))$}
\For{$i = 2, \dots, 2N-1$}
\For{$j = 0, \dots, d$} \Comment{$d$: maximum size of the stack}
 \State $\texttt{shift\_score} = \underset{\langle i-1, j+1 \rangle \xrightarrow{t} \langle i, j \rangle}{\max} \{ W[i-1, j+1] + \mathrm{log}\,\mathrm{score}(t) \}$
 \State $\texttt{reduce\_score} = \underset{\langle i-1, j-1 \rangle \xrightarrow{t} \langle i, j \rangle}{\max} \{ W[i-1, j-1] + \mathrm{log}\,\mathrm{score}(t) \}$
 \State $W[i, j] = \max(\texttt{shift\_score}, \texttt{reduce\_score})$
\EndFor
\EndFor
\State \Return $W[2N-1, 1]$ \Comment{pre-order: $W[2N-1, 0]$}
\EndProcedure
\end{algorithmic}
\end{algorithm*}

\begin{algorithm*}
\caption{Dynamic program for decoding with left- or right-dependent scores}
\label{alg:dp-dep}
\begin{algorithmic}[1]
\footnotesize
\Procedure{DP}{$\mathrm{score}$}
\State $W = \mathbf{0}$ \Comment{initialize the DP chart to zero}
\State $W[1, 1, \leftchild] = \mathrm{log}\,\mathrm{score}(\leftchild)$ \Comment{pre-order: $W[1, 2, \Leftchild(\SSS)]$} 
\For{$i = 2, \dots, 2N-1$}
\For{$j = 0, \dots, d$} 
\For{$t = 1, \dots, |\mathcal{T}|$}
 \State $\texttt{shift\_score}= \underset{\langle i-1, j+1, t' \rangle \xrightarrow{t} \langle \text{i, j, t} \rangle}{\max} \{ W[i-1, j+1, t'] + \mathrm{log}\,\mathrm{score}(t', t) \}$
 \State $\texttt{reduce\_score}= \underset{\langle i-1, j-1, t' \rangle \xrightarrow{t} \langle i, j, t \rangle}{\max} \{ W[i-1, j-1, t'] + \mathrm{log}\,\mathrm{score}(t', t) \}$
 \vspace{0.5em}
 \State $W[i, j, t] = \max(\texttt{shift\_score}, \texttt{reduce\_score})$
\EndFor
\EndFor
\EndFor
\State \Return $W[2N-1, 1, \Leftchild(\SSS)]$ \Comment{pre-order: $W[2N-1, 0, \rightchild]$}
\EndProcedure
\end{algorithmic}
\end{algorithm*}

\begin{table}[h]
\centering
\begin{tabular}{@{}lccc@{}}\toprule
Language & \multicolumn{3}{c}{\# Sentences} \\
\cmidrule(lr){2-4} 
& Train & Dev & Test \\ \midrule
English & 39832 & 1700 & 2416 \\
Basque & 7577 & 948 & 946 \\
French & 14759 & 1235 & 2541 \\
German & 40472 & 5000 & 5000 \\
Hebrew & 5000 & 500 & 716 \\
Hungarian & 8146 & 1051 & 1009 \\
Korean & 23010 & 2066 & 2287 \\
Polish & 6578 & 821 & 822 \\
Swedish & 5000 & 494 & 666 \\
\bottomrule
\end{tabular}
\caption{Dataset Statistics}
\label{tab:dataset}
\end{table}
\section{Dataset Statistics}
We use Penn Treebank for English and SPMRL 2013/2014 shared tasks for experiments on other languages. We further use available train/dev/test splits, the number of sentences in each split can be found in \Cref{tab:dataset}.

\section{Stack Size} \label{app:stack}
We empirically compare the stack size needed to parse English sentences using different linearizations in \cref{fig:stack}.
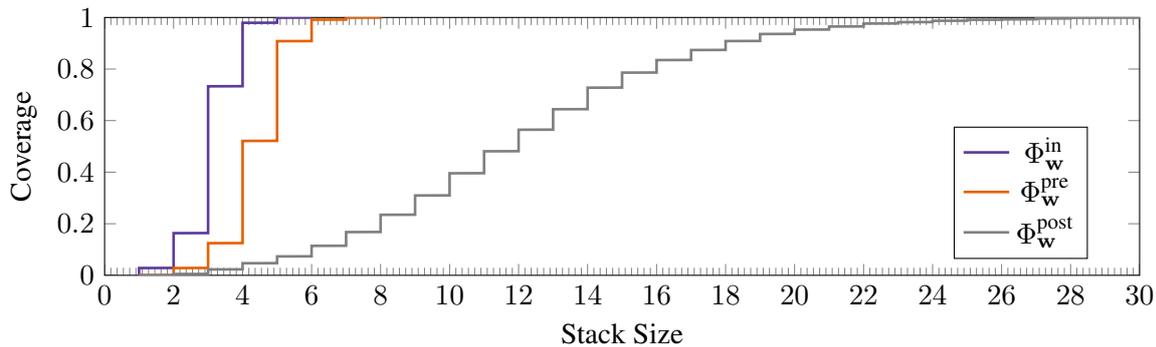
\begin{figure}[h]
\begin{tikzpicture}
\begin{axis}[
    xmin = 0, xmax = 30,
    ymin = 0, ymax = 1,
    xminorticks=true,
    minor x tick num=10,
	xlabel=Stack Size,
	ylabel=Coverage,
	width=0.95\columnwidth,
	height=5cm,
    legend style={at={(0.95,0.32)},anchor=east}
    ]

\addplot[color=prp, line width=0.35mm] table [x=tetra-x, y=tetra-y, col sep=comma] {figures/coverage.csv};

\addplot[color=orng, line width=0.35mm] table [x=td-sr-x, y=td-sr-y, col sep=comma] {figures/coverage.csv};

\addplot[color=gray, line width=0.35mm] table [x=bu-sr-x, y=bu-sr-y, col sep=comma] {figures/coverage.csv};


\legend{$\tetraTagger$, $\tdsrTagger$, $\busrTagger$}
\end{axis}
\end{tikzpicture}
\caption{With $\tetraTagger$ and stack size of 6, all the trees in WSJ test set can be linearized, thus the coverage is 1. On the other extreme, in order to cover all the trees with $\busrTagger$, we need a stack of size 29.}
\label{fig:stack}
\end{figure}

\section{Preprocessing}\label{app:preprocess}
Since our linearization works for binary trees, we do the following preprocessing: we collapse the unary rules by concatenating the nonterminal labels. We then binarize the tree using the \texttt{nltk} library. For in-order linearization, we first perform right-corner transformation of the tree and then run bottom-up shift reduce parsing on the transformed tree followed by the map and merge function introduced in \cref{app:theorems}.\footnote{We empirically test these map and merge functions, and verify that the sequence of transformed shift--reduce actions perfectly matches the original tetratags.}

\section{Experimental Setup} \label{app:exp}
We experiment on 3 \texttt{NVIDIA GeForce GTX 1080 Ti} nodes. The batch size is 32 and sentences in each epoch are sampled without replacement. We use gradient clipping at 1.0 and learning rate \texttt{3e-5} with a warmup over the course of 160 training steps. We calculate FMeasure of the development set 4 times per epoch and cut the learning rate in half, whenever the FMeasure fails to improve. We set the initial epochs to 20, however, after 3 consecutive decays in the learning rate, training is terminated. The checkpoint with the best development FMeasure is used for reporting the test scores. The stack depth for the decoding step is set to the maximum depth of the stack in the training set.
In experiments with the LSTM model, we use a two-layered BiLSTM network. We observe in \Cref{tab:lstm-layers} that adding more layers has a minimal effect on parsers' performance. 
\begin{table}[]
    \centering
    \begin{tabular}{@{}lccc@{}}\toprule
        & Recall & Precision & FMeasure \\ \midrule
        $\tdsrTagger$ + 2-layered BiLSTM & 93.76 & 94.5 & 93.9 \\
        $\tdsrTagger$ + 3-layered BiLSTM & 93.93 & 94.45 & 94.19 \\
        \bottomrule
    \end{tabular}
    \caption{The effect of increasing BiLSTM layers on parsers' performance.}
    \label{tab:lstm-layers}
\end{table}

\end{document}